%% file: main_arxiv.tex
\ifdefined\pdfoutput\pdfoutput=1\fi
\documentclass[10pt]{article}
\usepackage[utf8]{inputenc}
\usepackage[T1]{fontenc}
\usepackage{lmodern}
\usepackage[letterpaper,margin=1in]{geometry}
\usepackage{natbib}
\usepackage{microtype}

\input{math_commands.tex}

\usepackage{hyperref}
\usepackage{url}
\usepackage{enumitem}
\usepackage{amsmath,amssymb,amsthm}
\theoremstyle{plain}

\newtheorem{theorem}{Theorem}
\newtheorem{lemma}{Lemma}

\newtheorem{definition}{Definition}
\theoremstyle{remark}

\usepackage{etoc}
\usepackage{graphicx}
\usepackage{booktabs}
\usepackage{float}
\usepackage{enumitem}
\usepackage[most]{tcolorbox}
\tcbuselibrary{breakable}

\definecolor{boxOrange}{RGB}{204,122,0}
\newtcolorbox{defbox}[1]{
  enhanced, breakable,
  colframe=boxOrange, boxrule=0.35pt, arc=1mm,
  title={\textbf{#1}}, coltitle=black,
  fonttitle=\small\sffamily\bfseries,
  colbacktitle=boxOrange!20!white,
  colback=boxOrange!6!white,
  boxed title style={sharp corners, boxrule=0.35pt, colframe=boxOrange,
    top=1pt, bottom=0.5pt, left=4mm, right=4mm},
  attach boxed title to top left={xshift=4mm, yshift*=-1.2mm},
  boxsep=1.5mm, top=1.5mm, bottom=1.5mm, left=2.5mm, right=3mm,
  before skip=8pt, after skip=8pt
}

\definecolor{boxBlue}{RGB}{0,110,180}
\newtcolorbox{objbox}[1]{
  enhanced, breakable,
  colframe=boxBlue, boxrule=0.35pt, arc=1mm,
  title={\textbf{#1}}, coltitle=black,
  fonttitle=\small\sffamily\bfseries,
  colbacktitle=boxBlue!20!white,
  colback=boxBlue!6!white,
  boxed title style={sharp corners, boxrule=0.35pt, colframe=boxBlue,
    top=1pt, bottom=0.5pt, left=4mm, right=4mm},
  attach boxed title to top left={xshift=4mm, yshift*=-1.2mm},
  boxsep=1.5mm, top=1.5mm, bottom=1.5mm, left=2.5mm, right=3mm,
  before skip=8pt, after skip=8pt
}

\definecolor{boxPurple}{RGB}{126,63,171}
\newtcolorbox{ttlbox}[2]{
  enhanced, breakable,
  colframe=#1, boxrule=0.35pt, arc=1mm,
  title={\textbf{#2}}, coltitle=black,
  fonttitle=\small\sffamily\bfseries,
  colbacktitle=#1!20!white,
  colback=#1!6!white,
  boxed title style={sharp corners, boxrule=0.35pt, colframe=#1,
    top=1pt, bottom=0.5pt, left=4mm, right=4mm},
  attach boxed title to top left={xshift=4mm, yshift*=-1.2mm},
  boxsep=1.5mm, top=1.5mm, bottom=1.5mm, left=2.5mm, right=3mm,
  before skip=8pt, after skip=8pt
}

\definecolor{boxGreen}{RGB}{0,140,90}
\newtcolorbox{takeawaybox}[1]{
  enhanced, breakable,
  colframe=boxGreen, boxrule=0.35pt, arc=1mm,
  title={\textbf{#1}}, coltitle=black,
  fonttitle=\small\sffamily\bfseries,
  colbacktitle=boxGreen!20!white,
  colback=boxGreen!6!white,
  boxed title style={sharp corners, boxrule=0.35pt, colframe=boxGreen,
    top=1pt, bottom=0.5pt, left=4mm, right=4mm},
  attach boxed title to top left={xshift=4mm, yshift*=-1.2mm},
  boxsep=1.5mm, top=1.5mm, bottom=1.5mm, left=2.5mm, right=3mm,
  before skip=8pt, after skip=8pt
}

\newtcolorbox{questionbox}{
  enhanced, breakable,
  colframe=boxGreen, boxrule=0.35pt, arc=1mm,
  colback=boxGreen!6!white,
  boxsep=1.5mm, top=1.8mm, bottom=1.8mm, left=2.5mm, right=3mm,
  before skip=9pt, after skip=9pt
}

\definecolor{boxPink}{RGB}{199,62,125}

\usepackage{capt-of}
\usepackage{subcaption}
\definecolor{boxTeal}{RGB}{0,128,128}
\newtcolorbox{tablebox}[1]{
  enhanced, hbox,
  title={#1},
  colback=boxTeal!5, colframe=boxTeal, coltext=black, coltitle=white,
  fonttitle=\small\bfseries, arc=1mm, boxrule=1pt,
  fontupper=\small,
  boxsep=1pt, left=2pt, right=2pt, top=2pt, bottom=2pt,
  toptitle=1pt, bottomtitle=1pt,
  before skip=6pt, after skip=4pt
}

\usepackage{tikz}
\usetikzlibrary{positioning,arrows.meta,decorations.pathreplacing,calc}

\definecolor{linkblue}{RGB}{0,51,128}
\hypersetup{
  colorlinks = true,
  linkcolor  = linkblue,
  citecolor  = linkblue,
  urlcolor   = linkblue,
  filecolor  = linkblue
}

\newcommand{\AC}{\mathsf{AC}}
\newcommand{\TC}{\mathsf{TC}}

\newcommand{\NL}{\mathsf{NL}}

\newcommand{\poly}{\mathrm{poly}}

\newcommand{\CoT}{\mathrm{CoT}}
\newcommand{\STCON}{\textsc{STCON}}

\newcommand{\modelc}{base model}
\newcommand{\BT}{\texttt{<BT>}}

\hypersetup{pdftitle={Understanding the Benefits of Scaffolds through Computational Separations},
            pdfauthor={Jingchu Gai, Adam Block, Andrej Risteski}}
\AddToHook{env/figure/begin}{\RenewDocumentCommand{\vspace}{s m}{}}
\title{Understanding the Benefits of Scaffolds\\ through Computational Separations}

\author{%
  \setlength{\tabcolsep}{9pt}%
  \begin{tabular}[t]{@{}ccc@{}}
    Jingchu Gai & Adam Block & Andrej Risteski \\
    \normalsize\begin{tabular}[t]{@{}c@{}}Machine Learning Department\\ Carnegie Mellon University\end{tabular} &
    \normalsize\begin{tabular}[t]{@{}c@{}}Departments of Computer Science\\ and Electrical Engineering\\ Columbia University\end{tabular} &
    \normalsize\begin{tabular}[t]{@{}c@{}}Machine Learning Department\\ Carnegie Mellon University\end{tabular} \\
    \normalsize\texttt{jgai@cs.cmu.edu} & \normalsize\texttt{adam.block@columbia.edu} & \normalsize\texttt{aristesk@cs.cmu.edu}
  \end{tabular}%
}
\date{}

\begin{document}

\maketitle

\begin{abstract}
Inference-time harnesses, i.e., external procedures that call a model many times and keep state across
calls, can substantially improve model capabilities. In this paper, we investigate the benefits of using such harnesses and the extent to which these benefits can be internalized through training.  More precisely, we conceptualize a harness $\mathcal{A}$ as an arbitrary wrapper around a base model $M$ and attempt to distill a student model from the harness's traces.  Empirically we find that, even with simple scaffolds and small models, the inference-time harness significantly outperforms a distilled student.  We then prove that such a gap is due to a fundamental representational limitation: inference-time harnesses can compute a closure of their per-call primitive that is
$\NL$-complete even when the primitive is a trivial $\AC^0$ lookup, and that no fixed-depth
forward pass can represent under the widely believed conjecture that $\NL\not\subseteq\TC^0$.  Our theory suggests that expanding the representational capacity of the student relative to the task at hand, e.g., by enlarging the model, making the task easier, or lengthening its chain-of-thought, is necessary to close the gap, which we verify empirically. Finally, we scale to an agentic harness on competition math, where the same gap reappears. Overall, our results show that a harness confers on the base model a profound gain in representational power, one that training a same-class student on its traces cannot transfer back into the weights.

\end{abstract}

\etocdepthtag.toc{mainmatter}

\section{Introduction}

Much of the recent progress on challenging reasoning and agentic tasks depends on the workflow surrounding a language model. An \emph{inference-time harness} (or \emph{scaffold}) keeps the underlying model fixed, but calls it repeatedly, maintaining state and choosing later queries based on earlier answers. Harnesses range from sampling and selecting complete solutions \citep{wang2023selfconsistency,cobbe2021verifiers} to searching over partial solutions \citep{yao2023tot}, backtracking with a verifier \citep{rohatgi2025vgb}, coordinating agents \citep{du2023debate}, and incorporating tools and environment feedback \citep{yao2023react}.

There is something puzzling about how elementary some of these strategies are---and why they should improve the performance of already strong models. For example, early systems achieving IMO gold-level performance \citep{huang2025winning}  and autonomous mathematics research systems like Aletheia \citep{feng2026autonomous} use designed, but simple loops of maintaining a candidate set of solutions, expanding them, and pruning them by verification occasionally. Why should a model capable of sophisticated mathematical reasoning benefit so much from having such a simple strategy imposed from outside?

Moreover, considering that simplicity of design is often desirable, the fact that carefully designed harnesses are needed to achieve strong performance raises the question of whether one can expect a model to learn such strategies on its own, thereby internalizing the benefits of a harness.
Indeed, \citet{cundy2024sequencematch} train models to perform backtracking themselves and \citet{gandhi2024sos} to search by learning from serialized search traces. It is unclear, however, when one can expect this internalization to be successful and when a practitioner would be better off simply using a harness at inference-time.

There is likely more than one reason why a harness can improve performance: for example, it may increase the effective capacity of the underlying model, provide a form of error correction, or provide implicit decomposition of the target task into semantically meaningful subtasks \citep{zhang2026mgh}. While these intuitions are plausible, turning them into a predictive theory that can guide harness design requires precise hypotheses that can be tested quantitatively and, ideally, proved under suitable assumptions. The prevalence and success of harnesses make this an increasingly important task.

In this paper, we examine the success of inference-time harnesses through a computational lens: we show a model capable of executing each individual step of a harness may nevertheless be unable to represent the computation performed by the composition represented by the harness. We first investigate this hypothesis experimentally, then give a theoretical setting in which such internalization is \emph{provably} impossible under standard computational modeling assumptions of the base model.

We first perform controlled experiments on Dyck languages (a formal language of valid nested parentheses) and Sokoban (a maze-like puzzle game), where we can train a model from scratch and vary the difficulty of generalization. We train a base model on ground-truth demonstrations, then augment it with a harness: either a principled backtracking algorithm (VGB) introduced by \citet{rohatgi2025vgb}, or A* search with model-pruned expansions \citep{hart1968astar,lehnert2024searchformer}. Together, the harness and base model specify a teacher whose behavior we attempt to reproduce with a student of the same size and model class as the base model. The student is trained on the teacher's traces, either from scratch or by fine-tuning the base model. In both cases, a substantial gap remains, particularly out of distribution. Changing the amount and composition of training data, the trace format, or the number of training iterations does not close this gap in the settings we test.

Why does this training fail to reproduce the teacher's behavior? One possible explanation is that a harness fundamentally increases the representational capacity of the base model, so that
reproducing the harness requires a larger model than executing its individual calls.
We investigate this by increasing the student's size while keeping the teacher fixed: sufficiently large students match and eventually exceed their teacher's performance. We also test a complementary hypothesis that a simpler harness should be easier for the original student to internalize. Specifically, we replace verifier-controlled backtracking with a rule that deletes a fixed number of tokens and regenerates. The same-sized student can match this simpler teacher.

Model size is not the only reasonable consideration: a student may also need more ``reasoning steps'' to carry out the harness's computation.  Indeed, recent work has shown that increasing the length of a chain of thought can be understood to mimic the representational capacity of a much deeper model \citep{merrill2024expressive,li2024chain,merrill2025depth,amiri2025lower}.  In particular, on modular arithmetic, we find that training a model to use a longer chain of thought lets it match and eventually exceed the accuracy obtained by augmenting a shorter-chain model with verifier-guided search. These experiments suggest that successful internalization depends on the student's size, the number of reasoning steps it can take, and the complexity of the strategy it must reproduce.

These observations motivate a formal separation in \emph{representational capacity}. We consider directed $s$--$t$ reachability \citep{jones1975space}, the task of path-finding in a graph. To model a search harness, we think of vertices as reasoning states and edges as valid transitions. The base model only answers local edge queries, which are adjacency-matrix lookups computable by constant-depth Boolean circuits ($\AC^0$). The harness maintains a search frontier and visited set, using these queries to perform breadth-first search and decide reachability in polynomial time.

Internalizing the harness requires the student to decide reachability from the graph itself. Under standard precision assumptions, fixed-depth transformers of polynomial size can be simulated by constant-depth threshold circuits ($\TC^0$), even with logarithmic-length chain-of-thought \citep{li2024chain,merrill2025depth}. Since directed reachability is $\NL$-complete, the widely believed computational complexity separation $\NL\not\subseteq\TC^0$ implies that no such student can reproduce the harness on all inputs. The separation thus captures a concrete obstacle to internalization: composing local decisions through an external search procedure can exceed the students' representational capacity.

Finally, we examine whether the same pattern persists with an agentic harness around a pretrained language model. We use Branching \citep{bateni2025branching}, which organizes model calls into a tree: leaves propose solutions independently, and internal nodes synthesize new solutions from their children. Using Qwen3-1.7B on competition mathematics (DAPO-Math-17k), we flatten the teacher's runs into transcripts and LoRA fine-tune the same backbone to reproduce them in a single autoregressive rollout. As the branching factor $d$ and depth $L$ increase, the teacher's accuracy improves while the student's eventually plateaus, widening the gap. Increasing the LoRA rank yields only modest gains, whereas enlarging the backbone closes the gap. These results support the same interpretation as our controlled experiments and theory: internalizing a harness may require a larger model than the one used in the harness.

\begin{figure}[t]
\centering
\vspace{-25pt}
\includegraphics[width=\linewidth]{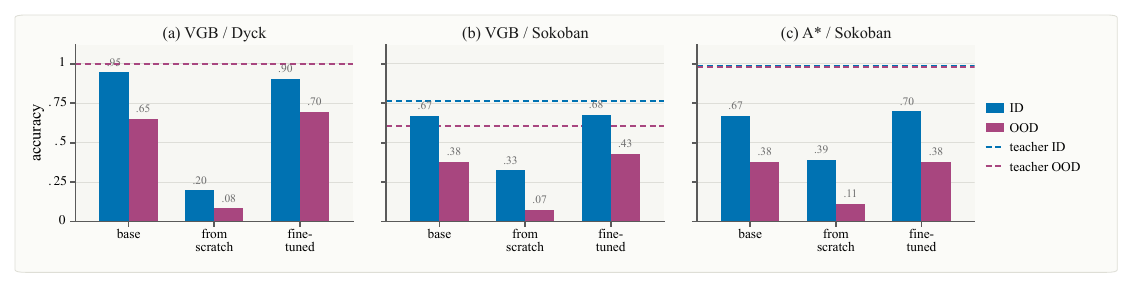}
\vspace{-20pt}
\caption{\textbf{Same-size students cannot recover the teacher.} \emph{ID (blue) / OOD (magenta) accuracy of the base model and the two students; dashed lines denote the teacher.}}
\vspace{-15pt}
\label{fig:main-results}
\end{figure}

\subsection{Our contributions}

\smallskip
\noindent\textbf{Students cannot internalize natural inference-time harnesses if they are the same size as the teacher.} We generate
training data from traces of a model $M$ used inside a harness, and perform supervised finetuning (SFT) of $M$ on them. Across two natural harnesses: a principled backtracking algorithm (VGB)
\citep{rohatgi2025vgb} and A* search \citep{hart1968astar}; a set of natural, extensive ablations varying the amount of traces, training iterations, trace format, amount of out-of-distribution data---the resulting student never closes the gap to its teacher.

\smallskip
\noindent\textbf{The performance gap closes when increasing relative expressive power of the student.} We show empirically that the
gap in performance between a model with a harness and a (finetuned) model without the harness closes as the student's expressive power grows relative to the task---for example, using a larger student, an
easier scaffold, or a longer chain-of-thought.

\smallskip
\noindent\textbf{There is a formal obstruction  to internalizing a harness provided by representational capacity.} We consider directed $s$--$t$ reachability on a graph as a paradigmatic problem modeling path-finding. A base model that answers local edge queries---an $\AC^0$ computation---can be used by a polynomial-time harness to run breadth-first search and decide reachability. On the other hand,  a student with an $O(\log n)$-length chain-of-thought provably stays inside $\TC^0$. By contrast, under standard precision assumptions, a fixed-depth, polynomial-size student remains in $\TC^0$ even with an $O(\log n)$-length chain-of-thought. Thus, assuming $\NL\not\subseteq\TC^0$, no such student can internalize the harness.

\smallskip
\noindent\textbf{Our results extend to popular harnesses on mathematical reasoning tasks.}
We distill the Branching harness \citep{bateni2025branching} on competition math using Qwen3-1.7B. The harness reaches $43.6\%$ accuracy with $85$ model calls, while a single-pass student saturates near $36\%$. Increasing the LoRA rank improves the student by only $2.2$ percentage points; increasing its backbone from $1.7$B to $4$B closes the gap.

\subsection{Preliminaries and notation}

Throughout the paper we fix a base autoregressive model $M$ (the \emph{base model}), trained
on clean demonstrations of a task with standard next-token prediction. An
\emph{inference-time harness} (or \emph{scaffold}) $\mathcal{A}$ is any algorithm that can call $M$ as an oracle. The resulting teacher is $\mathcal A[M]$. A run of $\mathcal{A}$ produces a \emph{trace}---the
full sequence of calls and answers. The specific form of this trace will vary with the choice of scaffold. A \emph{student} $S$ is a model trained by next-token prediction on
these traces. We ask whether $S$ can \emph{internalize} $\mathcal{A}$---i.e.\ match the accuracy of
$\mathcal{A}[M]$ from one forward pass.

\noindent\textbf{Computational complexity classes.} We recall several standard definitions of uniform circuit classes; see
\citet{arora2009complexity} and \citet{vollmer1999circuit} for background. \(\AC^0\) denotes constant-depth, polynomial-size circuits
with unbounded fan-in AND/OR/NOT gates. \(\TC^0\) adds unbounded
fan-in threshold (majority) gates. Throughout the
formal analysis, these families are logspace-uniform: a deterministic procedure
generates the circuit from $1^n$ using $O(\log n)$ workspace.  The classes
$\mathsf L$ and $\NL$ consist of languages decided using deterministic and
nondeterministic logarithmic space, respectively. We recall the following standard inclusions:
\vspace{-1.5mm}
\[
    \textstyle
    \AC^0 \subseteq \TC^0 \subseteq \mathsf L \subseteq \NL \subseteq \mathsf P .
\]
\vspace{-5mm}

\section{Students with the Same Expressive Power as the Base Model Cannot Recover the Teacher's Performance}
\label{sec:cannot-recover}

We begin with empirical observations derived in simple, but clean settings in which we can train a student model from scratch, and perform thorough sweeps and ablations across various relevant aspects of the setup. Concretely, we consider harnesses inspired by search algorithms (principled backtracking based on \citep{rohatgi2025vgb} and A*), for the tasks of generating valid strings of nested parentheses (i.e. Dyck languages) and Sokoban (a maze-like puzzle game) respectively---and ask whether distillation can transfer the harness's gains to a student of the same size and architecture as the base model, and hence the same expressive power. We find that students trained from scratch perform far below even the base models; students fine-tuned from the base models also cannot reach the teachers' gains.

\begin{figure}[t]
\centering
\includegraphics[width=\linewidth]{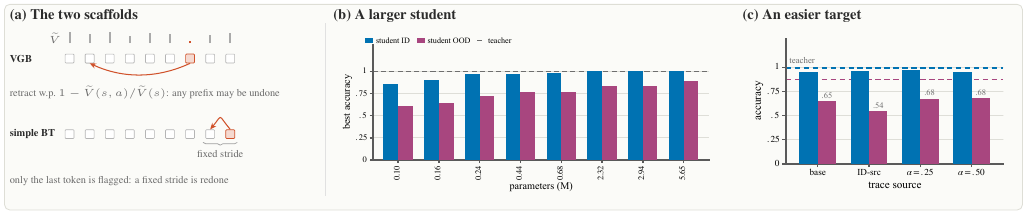}
\setlength{\abovecaptionskip}{4pt}
\small
\vspace{-10pt}
\caption{\textbf{Bigger students and a simpler scaffold narrow the gap.} \emph{(a) The two scaffolds. (b) Accuracy against the student's parameter count. (c) Accuracy against the distribution the traces were drawn from; dashed lines mark the teacher.}}
\vspace{-15pt}
\label{fig:scale-bt}
\end{figure}

\subsection{Setup}
We use two tasks with an easy-to-control in-distribution (ID) / out-of-distribution
(OOD) instances.
\begin{enumerate}[leftmargin=*]
\item \textbf{Dyck}: sequences of properly nested brackets, with every opening bracket matched by a closing bracket of the same type. We use $4$ bracket types and sequences of length $64$. The ID distribution draws bracket types with probabilities $[0.1,0.2,0.3,0.4]$, while the OOD distribution reverses these probabilities to $[0.4,0.3,0.2,0.1]$.
\item \textbf{Sokoban}: a puzzle in which a player moves through a grid and pushes a box onto a target, with walls blocking the movement. We use a $5\times5$ grid with one player, one box, one target, and up to three walls, with actions $\{U,D,L,R\}$. ID instances have short solutions, while OOD instances are biased toward longer ones (reversed solution-length buckets).

\end{enumerate}

We consider the following two scaffolds:

\begin{enumerate}[leftmargin=*]
    \item \textbf{VGB} \citep{rohatgi2025vgb}. The \modelc{} proposes an action
$a\sim M(\cdot\mid s)$, a verifier scores it, and the step is committed or retracted (\BT)
according to the Metropolis-style ratio
\vspace{-1.5mm}
\[
  \textstyle p_{\text{retract}}
  \;=\;\max\big(\varepsilon,\;1-\widetilde V(s,a)/\widetilde V(s)\big).
\]
\vspace{-5mm}

\noindent Here $\widetilde V$ is the \emph{policy-conditional} success probability---how likely
the \modelc{} is to produce a valid solution from the current prefix---so the ratio measures the
progress the proposed action makes:
\vspace{-1.5mm}
\[
  \textstyle \widetilde V(s)=\mathbb{E}_{\tau\sim M(\cdot\mid s)}
  \big[\mathbf 1[\text{valid}(s\circ\tau)]\big],
  \qquad \widetilde V(s,a):=\widetilde V(s\circ a).
\]
\vspace{-5mm}
\item \textbf{A* search} \citep{hart1968astar} maintains a frontier of candidate states and repeatedly expands the state with the smallest $f=g+h$. Here $g$ is the number of moves taken so far, and $h$ is the Manhattan distance from the box to the target. At each expansion, we retain only the top-$k{=}2$ successor actions according to \modelc's logits. The trace records when a node enters the frontier (\texttt{<CREATE>}) and when it is expanded (\texttt{<CLOSE>}), along with its $g,h,f$ values; it ends with the resulting \texttt{<PLAN>}. This search-log serialization follows \citet{lehnert2024searchformer}.

\noindent The trace prepends a quantized verifier bucket before each action or \BT{} event, so
the student sees the signal that drove every backtrack.

\end{enumerate}

\begin{figure}[t]
\centering
\vspace{-15pt}
\includegraphics[width=\linewidth]{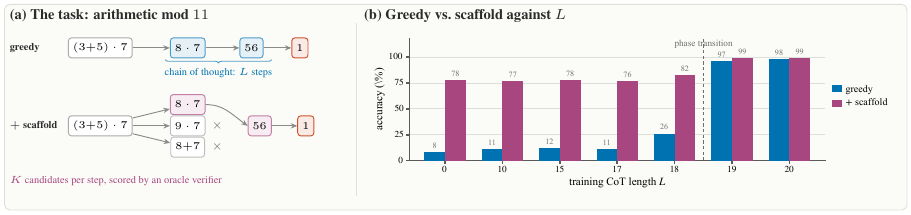}
\setlength{\abovecaptionskip}{4pt}
\vspace{-10pt}
\small
\caption{\textbf{A long enough chain-of-thought matches the scaffold.} \emph{(a) The task. (b) Accuracy against the training chain-of-thought length $L$; the dashed line marks the phase transition.}}
\label{fig:cotlen}
\vspace{-15pt}
\end{figure}

We defer the exact Metropolis derivation, verifier instantiations, and trace formatting
to Appendix~\ref{app:scaffold-details}.

Both student and teacher models are $2$-layer Transformers with
$d_{\text{model}}{=}128$ and $4$ heads ($\approx0.68$M parameters), so the student is the
\emph{same size} as the model the teacher. The \modelc{} is trained on clean ID
demonstrations. Students are trained in two ways: (i) \emph{from scratch} on a pool of traces and (ii) \emph{fine-tuned} from the \modelc{} on a mixture of traces and clean data. We
report accuracy on held-out ID and OOD splits under greedy/top-$p$ decoding (\texttt{gen});
unless otherwise noted, the teacher figure is the full scaffold $\mathcal{A}[M]$.

\subsection{Experimental Results}
\label{sec:cannot-recover-results}

We first show that in both tasks, the student cannot recover the teacher's accuracy. We then perform  ablations on the training data
distribution, the number of training iterations, and the amount of training data that gives strong evidence that
the gap between the student and the teacher is not caused by training dynamics.

\noindent\textbf{Student Cannot Recover Teacher Model.}
Figure~\ref{fig:main-results} compares the \modelc, the teacher $\mathcal{A}[M]$, and the two
students in the three (algorithm, domain) settings. A student trained from scratch on the
teacher's traces falls below the \modelc{} both ID and OOD, likely because the traces are a
much harder learning target than clean task data. A student fine-tuned from the \modelc{} stays
close to the \modelc{} on both splits, but well below the teacher.

\noindent\textbf{Ablations.}\label{sec:cannot-recover-ablation}
We give strong evidence the gap is not due to a failure of training: we sweep various relevant aspects in the training setup, and the same-size student always stays
below the teacher. Full grids, including four trace formats, are in
Appendix~\ref{app:cannot-recover-ablation}.

\noindent$\bullet$~\emph{Fine-tuning sweeps.}
Starting from the \modelc{} we vary, one axis at a time: the \emph{clean-data ratio} mixed with
VGB traces ($0$--$100\%$), the number of \emph{fine-tuning iterations} ($3$K--$8$K), the
\emph{number of traces} ($1$K--$20$K traces), and the \emph{training-source mixture}
$P_\alpha=(1-\alpha)P_{\text{ID}}+\alpha P_{\text{OOD}}$ of the puzzles the traces are drawn
from ($\alpha\in\{0,0.25,0.5,0.75\}$). Across all four axes the student tracks the base
model and stays below the teacher dashed line (Figure~\ref{fig:vgb-ft-ablation},
Appendix~\ref{app:ft-grids}). Adding more data, more iterations, and richer source mixtures change
the student only marginally.

\noindent$\bullet$~\emph{From-scratch sweeps.}
Holding the student at the fixed teacher scale ($0.68$M), we sweep the remaining data and
optimization knobs of the from-scratch training: random \emph{seeds}, the training-source
$\alpha$-mix, and an \emph{OOD-source} variant where the training traces themselves come
from OOD instances. Across all of these sweeps, the gap between the student and the teacher remains
(Figure~\ref{fig:app-fs}, Appendix~\ref{app:fs-grids}).

\section{The Student--Teacher Gap Closes as the Student's Expressive Power Increases}
\label{sec:closing-gap}

In Section~\ref{sec:cannot-recover} we systematically performed various changes to the training setup, but the student was always taken to be in the
\emph{same} model class as the model its teacher wraps.
In this section, we remove that constraint and investigate the effect of allowing the student's expressive
power to grow---relative to the hardness of the task. We find that the gap
is not an intrinsic property of trace distillation: it shrinks, and eventually vanishes, along
many axes that increase the student's representational capacity.
We consider three ways of increasing the student's expressive power relative to the scaffolded
teacher: enlarging the student's size (Section~\ref{sec:gap-size}), replacing the scaffold with a simpler
one (Section~\ref{sec:gap-easy}), and lengthening the student's chain-of-thought
(Section~\ref{sec:gap-cot}).

\subsection{Bigger Students Learn Harnesses Better}
\label{sec:gap-size}

We first consider increasing the student's expressive power by directly increasing its number of
parameters.

\noindent\textbf{Setup.}
Section~\ref{sec:cannot-recover} fixed the student at the teacher scale ($\approx0.68$M).
Here we sweep $12$ scales over two orders of magnitude in parameter count
($0.10$M--$11.5$M) and ask whether a \emph{bigger} student internalizes the scaffold better.
Each (scale, clean-ratio, iteration) cell is a fresh student trained on the same pool of traces
as Section~\ref{sec:cannot-recover}, evaluated on the held-out ID/OOD split. We report the
per-scale best accuracy.

\noindent\textbf{Results.}
As the student grows, its gap to the teacher shrinks on both ID and OOD: on VGB-Dyck
(Figure~\ref{fig:scale-bt}b) and on A*- and VGB-Sokoban (Figure~\ref{fig:app-modelsize-best},
Appendix~\ref{app:validation-size}).

\subsection{Simpler Harnesses Can be Learned without Increasing Student Size}
\label{sec:gap-easy}

Next, we consider making the target task easier: the teacher's scaffold is
replaced by a simpler one, which implicitly increases the student's expressive power relative to the
scaffolded teacher. We again find that the gap between the student and the teacher shrinks.

\noindent\textbf{Setup.}
VGB and A*-search
both maintain unbounded search state---the number of steps the teacher takes and the amount of
history grow with the problem. If we instead consider a harness which is \emph{bounded and local}, the student becomes larger in a relative sense. We test the simplest such scaffold on Dyck: single-step backtracking
(\textbf{BT}), where a verifier
flags the last token as erroneous and the model retracts a fixed stride and regenerates. The
distilled student learns to emit its \emph{own} \BT{} tokens autoregressively---internalizing
the policy with no external verifier---and we vary the source distribution of the distilled
traces, an $\alpha$-mixture between the ID and OOD bracket distributions.

\noindent\textbf{Results.}
On ID, the student recovers the teacher's accuracy (Figure~\ref{fig:scale-bt}c), whereas a
same-size student distilled from VGB on Dyck stays below its teacher
(Section~\ref{sec:cannot-recover}, Figure~\ref{fig:main-results}). On OOD, the student does not
fully recover the teacher's accuracy, and how close it gets depends on the distribution the
traces are drawn from: traces from the ID distribution alone recover little, while mixing in a
small fraction of OOD traces ($\alpha{=}0.25$) brings the student much closer to the teacher
(Figure~\ref{fig:scale-bt}c). With this choice of training distribution, the gap between student
and teacher is smaller than when the student learns VGB. The full grid, including single-pass
accuracy and synthetic-BT fine-tunes, is in Table~\ref{tab:app-bt} (Appendix~\ref{app:validation-easy}).

\subsection{Longer Student Chains-of-Thought Allow for Better Performance}
\label{sec:gap-cot}

Prior theoretical work has theoretically shown that increasing the length of the chain-of-thought also
increases a student's expressive power \citep{merrill2024expressive,li2024chain}. In this section, we consider a modular arithmetic task and show that the longer the student's
chain-of-thought, the smaller the gap to the scaffolded teacher.

\noindent\textbf{Setup.}
We consider modular arithmetic over $\mathbb Z_{11}$ \citep{feng2023cot}, built up from
numbers, arithmetic operators, and parentheses. A reduction step evaluates the
leftmost subexpression that is ready to be computed and replaces it with its
value modulo $11$, until we are left with a single number. This rule
gives each expression a prescribed chain of $n_{\text{ops}}$ steps. We train a
depth-$3$ transformer to produce only the first $L$ steps of that chain and
vary $L$ to control how much of the computation it learns to perform. At
inference time, the scaffold samples $K{=}16$ candidate next steps from the
model. An oracle checks which candidates match the next valid reduction; the
scaffold commits a valid candidate, or backtracks if all $16$ fail.

\noindent\textbf{Results.}
Figure~\ref{fig:cotlen}b ($n_{\text{ops}}{=}20$) compares, at each training chain-of-thought
length $L$, the model decoded greedily with the same model used in the scaffold. With the
scaffold (``+ scaffold'' bars), accuracy stays high at every $L$. Without it (``greedy'' bars), a
model trained on a short chain-of-thought falls far below the scaffold; its accuracy rises with
$L$, sharply across the phase transition marked by the dashed line between $L{=}18$ and
$L{=}19$, and at $L{=}19$ and $L{=}20$ it matches the scaffold.
The experiment
therefore confirms that increasing the student's expressive power through a longer chain-of-thought
closes its gap to the scaffolded teacher. The task definition, scaffold, and the
$n_{\text{ops}}{=}10$ curve are in Appendix~\ref{app:validation-cot}.

\section{An Expressivity Separation Between Scaffold and Student}
\label{sec:theory-gap}

The experiments in Sections~\ref{sec:cannot-recover} and~\ref{sec:closing-gap} empirically show that a
student trained on a scaffold's traces cannot match its teacher's accuracy when the student
is only as expressive as the teacher model. Once the student's expressive power grows
relative to the teacher's, however---through a larger model, a longer chain-of-thought, or an
easier task---the gap between the two narrows. The natural hypothesis is therefore that the
student cannot even \emph{represent} what the computation of the scaffold. This section shows this hypothesis is formally true under natural mathematical modeling abstractions from complexity theory used to understand the representational power of Transformers \citep{merrill2022saturated,li2024chain}.

\begin{figure}[t]
\centering
\vspace{-15pt}
\includegraphics[width=\linewidth]{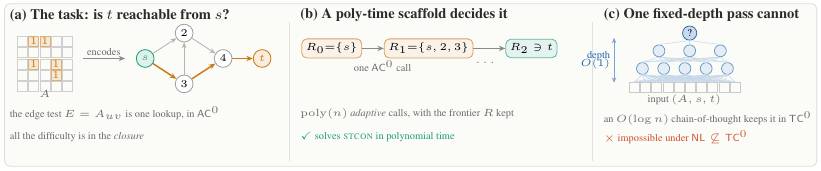}
\setlength{\abovecaptionskip}{4pt}
\vspace{-10pt}
\small
\caption{\textbf{Illustration of the directed \(s\)--\(t\) reachability task.} \emph{(a) The path-finding task. (b) With an $\AC^0$ edge oracle, a harness decides it in polynomial time. (c) A model with a bounded chain-of-thought cannot.}}
\label{fig:theory}
\end{figure}

\subsection{Setup: path-finding in a graph}

We use directed \(s\)--\(t\) reachability (\STCON, \citep{jones1975space}) as a paradigm mathematical abstraction for the task of path-finding in a graph. An instance consists of a directed graph with adjacency matrix \(A\in\{0,1\}^{N\times N}\), a source \(s\), and a target \(t\). The task is to decide whether there is a path \(s=v_0\to v_1\to\cdots\to v_T=t\). Checking whether one proposed move is valid in the graph is easy: the local edge predicate \(E_A(u,v)=A_{uv}\) is computable in \(\AC^0\). Indeed, \(E_A(u,v)=\bigvee_{i,j\in[N]}(A_{ij}\land[u=i]\land[v=j])\), which is a constant-depth circuit of polynomial size because each equality test is an AND of \(O(\log N)\) literals and the outer OR has \(N^2\) terms. Reachability asks whether some sequence of such moves leads to \(t\), so it requires combining many local answers.

Under standard complexity separation assumptions, there is a gap between the computational difficulties of checking for the validity of a single move and connectivity of $s$ and $t$: while checking an edge is in \(\AC^0\), deciding whether a path exists is \(\NL\)-complete even under \(\AC^0\) reductions \citep{immerman1987languages}. Hence, assuming \(\NL\not\subseteq\TC^0\), no \(\TC^0\) circuit can solve this task.
This is helpful because our lower bound uses an expressivity upper bound for the student: under standard precision and uniformity assumptions, a fixed-depth transformer can be simulated by a \(\TC^0\) circuit \citep{merrill2022saturated,merrill2023parallelism}. The scaffold uses the model only for the simpler query \(E(A,u,v)=A_{uv}\), while the student must decide reachability from the full input \((A,s,t)\). Thus, even if we grant the student every computation in \(\TC^0\), the assumed separation \(\NL\not\subseteq\TC^0\) rules out solving all reachability instances.

In fact, the same machinery will be able to handle even chain-of-thought of logarithmic length. To that end, we'll define mathematically a model allowing chain-of-thought. Fix a finite
alphabet $\Sigma$ containing a padding symbol $\bot$, and an integer $c\ge1$, and set
\(
  T_c(n)=\bigl\lceil c\log_2(n+1)\bigr\rceil.
\)
Shorter runs are padded to this length, and the final predicate ignores the padding.

\begin{definition}[CoT with $\TC^0$ verification]\label{def:cot}
A language $L$ lies in $\CoT_{T_c(n)}(\Sigma,\TC^0)$ if there are step
predicates $V_{n,i}(x,z_{<i},z_i)$, for $1\le i\le T_c(n)$, and a final predicate
$B_n(x,z)$ such that
\[
  \textstyle
  x\in L\quad\Longleftrightarrow\quad
  \text{there exists } z\in\Sigma^{T_c(n)} \text{ such that}\quad
  B_n(x,z)\land\bigwedge_{i=1}^{T_c(n)}V_{n,i}(x,z_{<i},z_i).
\]
The predicates have a common constant depth bound and a common polynomial size bound
in $n$. They are jointly logspace-uniform: one logarithmic-space procedure, given
$(1^n,i)$, generates the circuit for $V_{n,i}$, and another, given $1^n$, generates
the circuit for $B_n$.
\end{definition}

This abstraction includes deterministic autoregressive generation: $V_{n,i}$ checks
that $z_i$ is the token prescribed by the model on the preceding prefix, and $B_n$
checks whether the final answer is yes. Allowing any transcript satisfying these
checks can only enlarge the class of student computations.

\subsection{Harnesses Allow for Increased Expressivity}

First, we show that adding logaritmically bounded CoT results in a model still in $\TC^0$, using the same transcript-enumeration argument as
\citet[Theorem~4 and Appendix~E]{scaffold_merrill2025depth}:

\begin{lemma}[Short CoT remain in $\TC^0$]\label{lem:logcot}
For every fixed finite alphabet $\Sigma$ and integer $c\ge1$,
\[
  \textstyle
  \CoT_{T_c(n)}(\Sigma,\TC^0)\subseteq\TC^0.
\]
\end{lemma}

There are only $|\Sigma|^{T_c(n)}=\poly(n)$ possible transcripts. A circuit checks
each transcript in parallel and accepts if one passes every check, increasing the depth only by two. Appendix~\ref{app:logcot} contains the complete proof.

\begin{theorem}[Scaffold--student separation]\label{thm:main}
Assume $\NL\not\subseteq\TC^0$. For $\STCON$ on an $N$-vertex graph
given by its adjacency matrix $A$, the following hold:
\begin{enumerate}[topsep=0.2em,itemsep=0.1em,leftmargin=*]
  \item[(i)] The local predicate $E(A,u,v)=A_{uv}$ lies in $\AC^0$. Given $N,s,t$
  and oracle access to $E_A$, a deterministic polynomial-time breadth-first-search
  scaffold decides reachability using at most $N^2$ oracle calls and $O(N)$ bits
  of external state.
  \item[(ii)] For every fixed finite alphabet $\Sigma$ and integer $c\ge1$,
  \[
    \textstyle
    \STCON\notin\CoT_{T_c(n)}(\Sigma,\TC^0),
    \qquad T_c(n)=\lceil c\log_2(n+1)\rceil,
  \]
  where $n$ is the encoded input length and the predicates satisfy
  Definition~\ref{def:cot}.
\end{enumerate}
\end{theorem}

Part~(i) is the standard breadth-first-search algorithm. For part~(ii),
Lemma~\ref{lem:logcot} would put a successful student in $\TC^0$. Composing it
with the uniform $\AC^0$ reductions to $\STCON$ would then put every language
in $\NL$ in $\TC^0$, contradicting the assumption. The full proof is in
Appendix~\ref{app:separation}.

The theorem can be seen as giving an example of a natural task on which a scaffold simulating a search procedure combines elementary
model calls into a computation that no student in the stated class can reproduce on
all inputs, even with $O(\log n)$ chain-of-thought tokens. This is a representational obstruction to internalization. Commensurate with the experiments, longer chains of thought or increase in depth can change asymptotic expressive power
\citep{merrill2024expressive,scaffold_merrill2025depth}.

\section{Agentic Scaffolding}
\label{sec:agentic}

This section moves from simple, controlled tasks  (Dyck, Sokoban) to more realistic ones: a
\textbf{Branching} scaffold \citep{bateni2025branching} applied to competition mathematics problems. We show
that the gap between student and teacher widens as the number of calls grows, thereby demonstrating the broader applicability of our earlier theory and experiments.

\begin{figure}[t]
\centering
\vspace{-10pt}
\includegraphics[width=\linewidth]{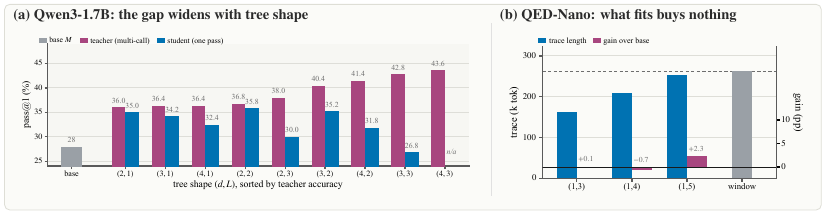}
\vspace{-15pt}
\caption{\textbf{Larger trees widen the gap.} \emph{(a) Teacher and student accuracy on Qwen3-1.7B per tree shape $(d,L)$. (b) Trace length and gain over the base model on QED-Nano.}}
\vspace{-15pt}
\label{fig:agentic-scaling}
\end{figure}

\noindent\textbf{Setup.}
The base model $M$ is Qwen3-1.7B (\texttt{enable\_thinking=False}); the task is DAPO-Math-17k,
with correctness judged by \texttt{math-verify} on the last \verb|\boxed{...}|. We distill
from a $1500$-problem pool of traces (cached pass@1 $\approx0.30$) and evaluate on a disjoint
$500$-problem set (pass@1 $\approx0.31$).

\noindent\textbf{Branching scaffold.}
The Branching scaffold calls $M$ in a depth-$L$ tree: the leaves call $M$ on the question at temperature
$1.0$, and each internal node calls $M$ once to synthesize a \emph{fresh} solution from its
children. The canonical $\mathbf{ks}{=}[3,3]$ tree uses $9$ leaves $\to3$ aggregators $\to1$
root ($13$ calls, no verifier). The distillation-knob ablations use the smaller
$\mathbf{ks}{=}[3]$ variant ($3$ leaves $+1$ root, $4$ calls); the shape sweep varies the topology of the full tree. We flatten the transcript
(\texttt{\#\# Attempt}\,/\,\texttt{\#\# Final Solution}) and LoRA fine-tune $M$ to emit the whole
trace, and hence the answer, in one autoregressive rollout, keeping only root-correct runs.

\noindent\textbf{Experimental Result.}
We distill Branching trees of increasing size into a same-size student, varying the number of
children $d\in\{2,3,4\}$ and the depth $L\in\{1,2,3\}$ (Figure~\ref{fig:agentic-scaling}a; full
grid in Figure~\ref{fig:app-branching-scaling}). When $d$ and $L$ are small (left of Figure~\ref{fig:agentic-scaling}a), the student recovers
most of the teacher's gain over the base model in one autoregressive rollout. This agrees with
Section~\ref{sec:gap-easy} (Figure~\ref{fig:scale-bt}c), where a simpler
backtracking rule narrows the teacher--student gap.
As $d$ and $L$ grow (right of Figure~\ref{fig:agentic-scaling}a), the teacher keeps improving as calls are added, while the unscaffolded student does not, and it
collapses once the tree reaches depth three. Thus, larger teacher trees are harder to reproduce with the tested same-size students.

\noindent\textbf{Ablation.} We compare changes to the student backbone, the number of
trainable parameters, and the data mixture (full results in Appendix~\ref{app:agentic-branching}).

\noindent$\bullet$~\emph{Backbone size has the larger effect.} We compare enlarging the backbone
with increasing LoRA rank and with full fine-tuning at fixed backbone size. Moving from the $1.7$B to the $4$B backbone closes the gap
and surpasses the teacher (Figure~\ref{fig:app-branching-capacity}a), whereas raising the LoRA
rank, or fine-tuning every weight, makes little difference
(Figure~\ref{fig:app-branching-capacity}b).

\noindent$\bullet$~\emph{The tested clean-data mixtures do not close the gap.} The student tolerates
about half clean data, then starts to fall back toward the clean-data baseline, and no mixture can reache the
teacher's performance (Figure~\ref{fig:app-branching-ablation}a).

\noindent\textbf{A harness on a strong prover relative to model size.}
Unlike the off-the-shelf Qwen3-1.7B used above, QED-Nano \citep{qednano2026} is a $4$B prover
built specifically for Olympiad proofs, trained on solutions from a $685$B model and then refined
by reinforcement learning, which makes it unusually strong relative to its size. After RL, the released checkpoint is used in a multi-call scaffold at evaluation
time, and its strongest reported results use that scaffold. We put Branching on that same
RL checkpoint and sweep the tree shape, scoring both arms of the same trace on $300$ problems
from the model's own RL corpus. Across the shapes that stay within seven calls in total, none achieves a gain
whose confidence interval is above zero (Figure~\ref{fig:agentic-scaling}b,
Appendix~\ref{app:qed}). A clear gain does appear once we go to $d{=}3$, $L{=}3$, but by then the
estimated total trace length is more than five times the context window considered
for distillation. Within this sweep, the shape with a clear gain also has the longest estimated trace.
This creates an obstacle to distilling the complete trace within that context budget. We do not train a QED-Nano student in this experiment.

\section{Conclusion}

We studied when a model can internalize the benefits of an inference-time harness. Across controlled tasks and agentic scaffolds, we find that same-size students remain below their teachers in our experiments, while larger students, longer chains of thought, and simpler scaffolds can narrow or close the gap. We formalize a representational obstruction that explains these observations: under $\NL\not\subseteq\TC^0$, a scaffold simulating a search procedure exceeds the expressive power of a short-CoT student class. We thus clarify computational benefits of scaffolds.

\subsection*{AI use statement}
We did not use AI to plan or design the experiments in this work. However, we used AI tools for
coding and for writing assistance. We have reviewed all AI-assisted work and take
responsibility for the final content of this paper.

\subsection*{Acknowledgements}
A.R. and J.G. are supported in part by NSF awards IIS-2211907, CCF-2238523, IIS-2403275, an Amazon
Research Award, ONR award N000142512124, a Google Research Scholar Award, and an OpenAI
Superalignment Fast Grant.  A.B. was supported in part by a research grant from Schmidt Sciences.

\bibliographystyle{iclr2027_conference}
\bibliography{references}

\appendix
\setcounter{topnumber}{1}\setcounter{totalnumber}{1}
\renewcommand{\topfraction}{0.9}\renewcommand{\textfraction}{0.1}\renewcommand{\floatpagefraction}{0.9}
\makeatletter\setlength{\@fptop}{0pt}\makeatother
\etocdepthtag.toc{appendixmatter}
\etocsettagdepth{mainmatter}{none}
\etocsettagdepth{appendixmatter}{subsection}
\etocsettocstyle{\section*{Appendix: contents}}{}
{\linespread{1.5}\selectfont\tableofcontents}
\newpage

\section{Related Work}
\label{app:related}

\begin{figure}[t]
\centering
\includegraphics[width=\linewidth]{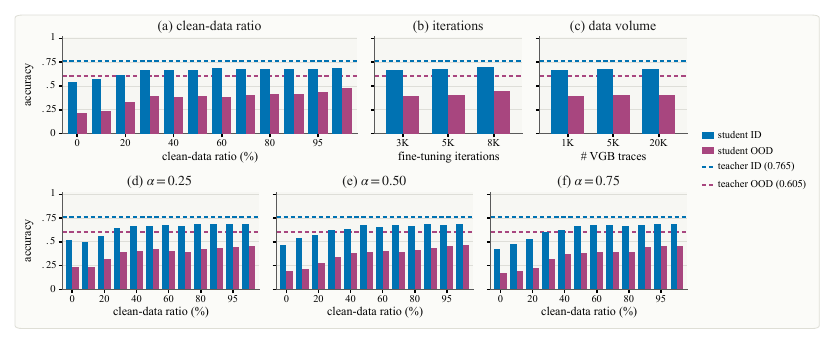}
\caption{\textbf{No fine-tuning setting reaches the teacher.} \emph{A same-size student fine-tuned on VGB traces (Sokoban, MC-oracle score, score tokens kept). Top row: clean-data ratio (a), number of iterations (b) and number of traces (c). Bottom row: the clean-data sweep repeated with training puzzles drawn from $P_\alpha$, for $\alpha=0.25$ (d), $0.50$ (e) and $0.75$ (f). Dashed lines mark the teacher.}}
\label{fig:vgb-ft-ablation}
\end{figure}

\begin{figure}[t]
\centering
\includegraphics[width=\linewidth]{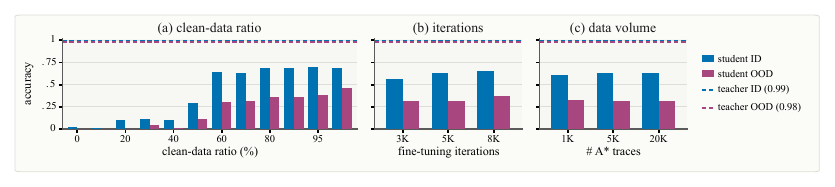}
\setlength{\abovecaptionskip}{4pt}
\small
\caption{\textbf{Fine-tuned A* students stay at the base model.} \emph{Fine-tuning the \modelc{} on A* traces (Sokoban). For every clean-data ratio, number of iterations and amount of data, the student stays close to the base model and well below the teacher (dashed).}}
\label{fig:app-astar-ft}
\end{figure}

\begin{figure}[t]
\centering
\vspace{-15pt}
\includegraphics[width=\linewidth]{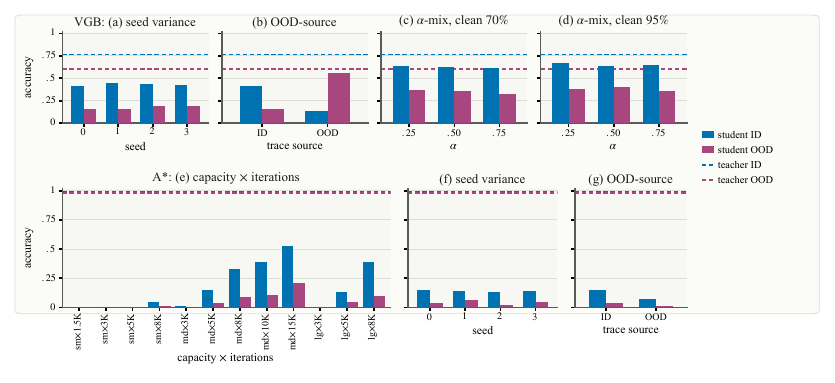}
\setlength{\abovecaptionskip}{4pt}
\small
\caption{\textbf{From-scratch students stay far below the teacher.} \emph{Students trained from scratch, with VGB in the top row and A* search in the bottom row. No seed, trace source or rough change in size brings either student close to its teacher (dashed; VGB $0.765/0.605$, A* $0.99/0.98$).}}
\label{fig:app-fs}
\end{figure}

\begin{table}[t]
\centering
\caption{\emph{The $12$ student scales of Section~\ref{sec:gap-size}, plus \texttt{large}. VGB is run on the $8$ scales from \texttt{tiny} to \texttt{m\_plus}; A* search and Dyck use all of them.}}
\label{tab:scales}
\setlength{\tabcolsep}{4.5pt}
\renewcommand{\arraystretch}{1.1}
{\footnotesize
\begin{tabular}{@{}lcccccc@{}}
\toprule
scale & \texttt{tiny} & \texttt{xsmall} & \texttt{small} & \texttt{s\_plus} & \texttt{teacher}
      & \texttt{t\_plus} \\
\midrule
$(d,L,H)$ & 32/2/4 & 48/2/4 & 64/2/4 & 96/2/4 & 128/2/4 & 160/2/8 \\
params    & $0.10$M & $0.16$M & $0.24$M & $0.44$M & $0.68$M & $0.98$M \\[3pt]
\midrule
scale & \texttt{mini} & \texttt{m\_m1} & \texttt{m\_m2} & \texttt{medium} & \texttt{m\_plus}
      & \texttt{large} \\
\midrule
$(d,L,H)$ & 192/3/8 & 224/3/8 & 256/3/8 & 256/4/8 & 320/4/8 & 384/6/8 \\
params    & $1.77$M & $2.32$M & $2.94$M & $3.73$M & $5.65$M & $11.5$M \\
\bottomrule
\end{tabular}%
}
\end{table}

\begin{figure}[t]
\centering
\vspace{-10pt}
\includegraphics[width=\linewidth]{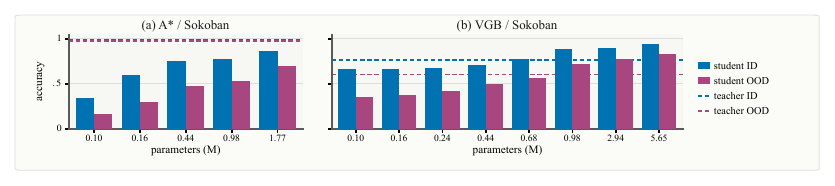}
\setlength{\abovecaptionskip}{4pt}
\small
\vspace{-15pt}
\caption{\textbf{Larger students close the gap on Sokoban.} \emph{Best accuracy at each scale in the two Sokoban settings.}}
\label{fig:app-modelsize-best}
\vspace{-10pt}
\end{figure}

\begin{figure}[t]
\centering
\includegraphics[width=\linewidth]{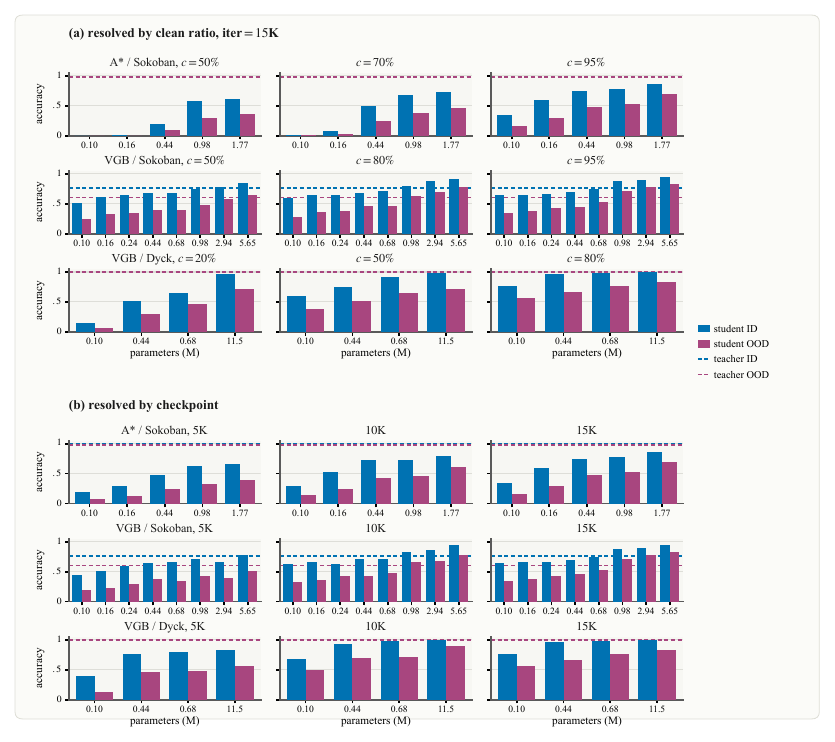}
\setlength{\abovecaptionskip}{4pt}
\small
\caption{\textbf{Accuracy increases with size at every clean ratio and checkpoint.} \emph{The sweeps of Figure~\ref{fig:app-modelsize-best} shown separately for each clean-data ratio and each checkpoint. Within each block, the rows are A*, VGB-Sokoban and VGB-Dyck from top to bottom.}}
\label{fig:app-modelsize-ratio}
\end{figure}

\begin{table}[t]
\centering
\caption{\emph{Full results of simple-BT distillation on Dyck, in one pass and with the student's own \BT{} tokens. Figure~\ref{fig:scale-bt}c in the main text uses the \BT{} rows of the distilled students.}}
\label{tab:app-bt}
\setlength{\tabcolsep}{4.5pt}
\renewcommand{\arraystretch}{1.12}
{\small
\begin{tabular}{@{}lcccccccc@{}}
\toprule
& \multicolumn{2}{c}{reference} & \multicolumn{3}{c}{distilled, trace source}
& \multicolumn{3}{c}{synthetic BT, errors from} \\
\cmidrule(lr){2-3}\cmidrule(lr){4-6}\cmidrule(lr){7-9}
& base & oracle-BT & ID & $\alpha{=}0.25$ & $\alpha{=}0.50$
& OOD & random & ID \\
\midrule
one pass, ID       & $0.947$ & ---     & $0.936$ & $0.957$ & $0.942$ & $0.947$ & $0.977$ & $0.840$ \\
one pass, OOD      & $0.650$ & ---     & $0.133$ & $0.470$ & $0.313$ & $0.250$ & $0.637$ & $0.450$ \\
\cmidrule(lr){1-9}
with \BT{}, ID      & ---     & $0.996$ & $0.966$ & $0.970$ & $0.957$ & $0.957$ & $0.963$ & $0.933$ \\
with \BT{}, OOD     & ---     & $0.873$ & $0.543$ & $0.677$ & $0.680$ & $0.637$ & $0.683$ & $0.697$ \\
\bottomrule
\end{tabular}%
}
\end{table}

\begin{figure}[t]
\centering
\includegraphics[width=\linewidth]{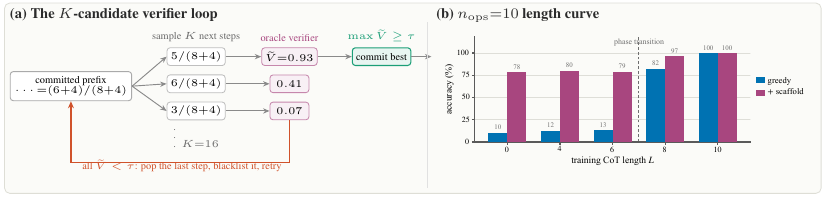}
\setlength{\abovecaptionskip}{4pt}
\small
\caption{\textbf{The same phase transition on an easier task.} \emph{(a) The verifier-backtracking harness with $K{=}16$. (b) Accuracy against the training chain-of-thought length $L$ for $n_{\mathrm{ops}}{=}10$.}}
\label{fig:app-cot}
\end{figure}

\begin{figure}[t]
\centering
\includegraphics[width=\linewidth]{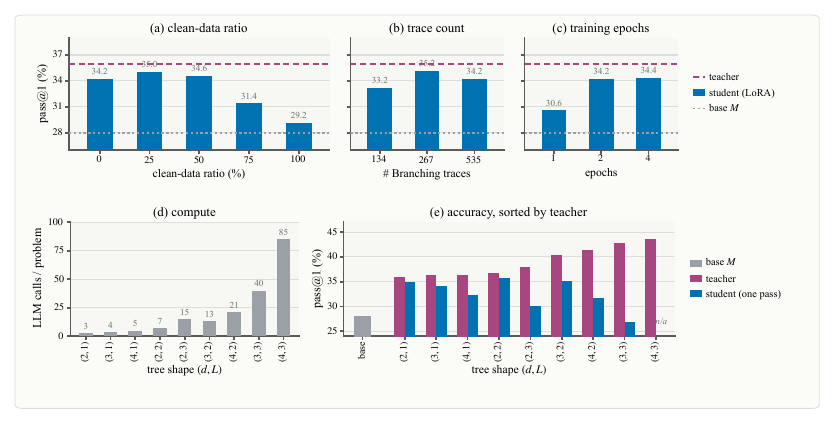}
\setlength{\abovecaptionskip}{4pt}
\small
\caption{\textbf{No distillation setting reaches the teacher, and only the teacher gains from more calls.} \emph{Branching distillation on Qwen3-1.7B. Top row: clean-data ratio (a), number of traces (b) and number of epochs (c); the teacher is dashed and the base model dotted. Bottom row: number of calls per problem for every tree shape $(d,L)$ (d), and accuracy with the shapes sorted by the teacher's accuracy (e).}}
\label{fig:app-branching-ablation}
\label{fig:app-branching-scaling}
\end{figure}

\begin{table}[t]
\centering
\caption{\emph{The five tree shapes we run and the median length of their full traces.}}
\label{tab:qed-shapes}
\setlength{\tabcolsep}{7pt}
\renewcommand{\arraystretch}{1.12}
{\small
\begin{tabular}{@{}lccccc@{}}
\toprule
shape & $(1,4)$ & $(1,5)$ & $(1,6)$ & $(2,2)$ & $(3,3)$ \\
\midrule
leaves        & $4$    & $5$    & $6$    & $4$     & $27$ \\
aggregations  & $1$    & $1$    & $1$    & $2{+}1$ & $13$ \\
calls         & $5$    & $6$    & $7$    & $7$     & $40$ \\
median trace  & $208$k & $253$k & $309$k & ---     & $\approx1.4$M \\
\bottomrule
\end{tabular}%
}
\end{table}

\begin{figure}[t]
\centering
\includegraphics[width=\linewidth]{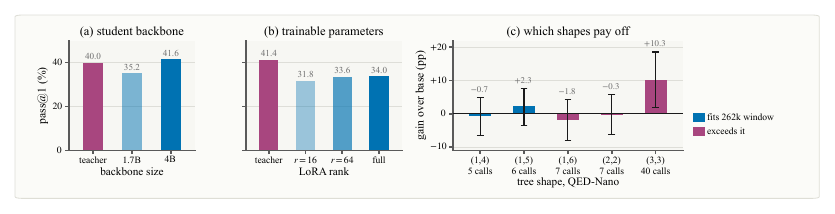}
\setlength{\abovecaptionskip}{4pt}
\small
\caption{\textbf{A larger backbone closes the gap.} \emph{(a) Student accuracy with the $1.7$B and $4$B backbones. (b) Student accuracy with LoRA ranks $16$ and $64$ and with full fine-tuning. (c) Gain of the harness over the base model on QED-Nano for each shape. Only the $40$-call shape has a gain whose interval excludes zero, and its trace is much longer than the context window.}}
\label{fig:app-branching-capacity}
\end{figure}

\noindent\textbf{Agents and harnesses.}
A lot of recent work on agents keeps the language model fixed and improves the harness around it.
Some harnesses let the model use tools and see feedback from an environment \citep{yao2023react},
or give it an interface for working with software \citep{yang2024sweagent}. Others ask the model to check and fix its own answers \citep{madaan2023selfrefine,shinn2023reflexion}, or let several copies of the model debate \citep{du2023debate}. A task can also be split into parts whose answers are then combined \citep{saha2024bsm}. Branching combines sampled solutions in a tree
\citep{bateni2025branching}. Strong results in math have also come from harnesses that let strong models generate and check solutions \citep{huang2025winning,feng2026autonomous,qednano2026}.
It is still not well understood why such simple harnesses help models that are already strong.
One idea is that a harness breaks the task into smaller, meaningful steps \citep{zhang2026mgh}. We
try to give clearer answers, using both experiments and theory. We also ask a different question.
Instead of building a better harness, we ask whether the gains from a harness can be moved into the weights of a student of the same size as the base model. We find that the model's expressive power can
stand in the way.

\noindent\textbf{Expressive power of transformers and chain-of-thought.}
Several papers use circuit complexity to describe what transformers can compute. A transformer
with a fixed number of layers and saturated or log-precision attention can only compute functions
in the class $\TC^0$ \citep{merrill2022saturated,merrill2023parallelism}. Writing out intermediate
reasoning steps lets a transformer compute more \citep{feng2023cot,li2024chain,merrill2024expressive},
and some tasks need long chains-of-thought \citep{amiri2025lower}. Most relevant to our theory,
\citet[Theorem~4]{scaffold_merrill2025depth} show that a chain-of-thought with $O(\log n)$ steps
keeps a transformer inside logspace-uniform $\TC^0$. They also show that a transformer whose depth
grows with the input can decide whether one node of a directed graph can reach another. Our
Lemma~\ref{lem:logcot} uses their argument, which goes through all possible chains-of-thought.
\citet{li2026wordram} study how a long enough chain-of-thought can run standard algorithms, such as
graph search. \citet{scaffold_mirtaheri2025} look at a setting close to our formal example. They
show that one polynomially long chain-of-thought can solve a connectivity task on undirected graphs
that polynomially many short chains, combined by majority vote, cannot solve. Their result assumes a
standard conjecture from circuit complexity. Our theorem puts together the known limit for short
chains-of-thought and the fact that directed reachability is hard. This gives a formal reason why a
student may fail to learn what a harness does. Our experiments study the same question in practice,
by distilling search and aggregation traces.

\noindent\textbf{Verifiers and inference-time search.}
A simple way to use more compute at inference time is to sample several solutions and pick one,
either with a verifier or by majority vote
\citep{cobbe2021verifiers,wang2023selfconsistency,brown2024monkeys}. Process reward models score
every step of a solution, which gives more detailed feedback than a verifier that only checks the
final answer \citep{lightman2024verify}. When test-time compute is used well, a small model can
sometimes beat a larger one \citep{snell2024testtime}. Other methods search over partial solutions. They use tree search guided by the model \citep{yao2023tot,hao2023rap}, classical planning such as A* search \citep{hart1968astar}, or backtracking guided by an imperfect verifier \citep{rohatgi2025vgb}. Another line of work teaches the model to do the search on its
own. These methods train on written-out search traces \citep{gandhi2024sos,lehnert2024searchformer},
add a backtracking action \citep{cundy2024sequencematch}, or distill the output of a harness into a single pass \citep{yu2024system2}. Our experiments also train on such traces. We find that a student of the same size as the base model does not reach its teacher. The gap gets smaller as the student becomes more expressive relative to the task. The work most similar to ours is
\citet{tsilivis2026internalization}, who study how a transformer can absorb an explicit procedure
into its weights. They have the model absorb its own chain-of-thought tokens, while we have it
absorb a harness that calls the model many times and keeps state between calls. They also prove
that gradient descent, with a curriculum that slowly removes the chain-of-thought, learns to compute
sparse parities in a single forward pass. We instead show a limit on what a same-size student can represent.

\noindent\textbf{Distillation.}
We conclude that a harness cannot be fully moved into the weights. This depends partly on one assumption: in our setting, what limits \emph{distillation} \citep{hinton2015distilling} is mainly the student's capacity. Many forms of distillation have been studied
\citep{agarwal2024policy,lu2025onpolicydistillation}. We use the simplest one, which is supervised
fine-tuning on the teacher's outputs. When capacity is not a constraint, this method is provably
optimal \citep{foster2024behavior}. Other papers show where it breaks down when capacity is limited
\citep{ross2011reduction,rohatgi2025computational,sriraman2026behavior,zhang2026does}. Together,
these results suggest that distilling a harness with SFT is a reasonable way to measure how much the
student's capacity holds it back.

\section{Supplementary Details for Section~\ref{sec:cannot-recover}}
\label{app:sec2}

In this appendix we describe the two harnesses and their traces, how we train the base models and the students, and all the ablations for Section~\ref{sec:cannot-recover}.

\subsection{Harness Details}
\label{app:scaffold-details}

We first describe the two harnesses from Section~\ref{sec:cannot-recover} and how we write their
traces for the student.

\subsubsection{VGB (probabilistic backtracking)}
\label{app:vgb}

At a prefix $s$, the \modelc{} proposes an action $a\sim M(\cdot\mid s)$, and the harness then either commits $a$ or emits a backtrack token \BT{}:
\[
    \textstyle
    p_{\text{retract}}
    \;=\;\max\!\Big(\varepsilon,\; 1-\tfrac{\widetilde V(s,a)}{\widetilde V(s)}\Big),
    \qquad
    \begin{cases}
    \text{commit } a & \text{w.p. } 1-p_{\text{retract}},\\[2pt]
    \BT{} & \text{w.p. } p_{\text{retract}}.
    \end{cases}
\]

\noindent\textbf{The score $\widetilde V$.}
Following Eq.~(3) of \citet{rohatgi2025vgb}, $\widetilde V(s)$ is the probability that the
\modelc{}, starting from $s$, finishes with a valid sequence:
\[
    \textstyle
    \widetilde V(s)\;:=\;\mathbb{E}_{\tau\sim M(\cdot\mid s)}\big[\mathbf 1[\text{valid}(s\circ\tau)]\big]\in[0,1],
    \qquad
    \widetilde V(s,a)\;:=\;\widetilde V(s\circ a).
\]
We estimate this probability in two ways:
\begin{itemize}
    \item \emph{MC oracle.} We sample $K$ completions $\tau_1,\dots,\tau_K$ of $s$ from the
    \modelc{} and take the fraction that are valid,
    $\widehat V_K(s)=\tfrac1K\sum_{k=1}^K \mathbf 1[\text{valid}(s\circ\tau_k)]$. We use $K{=}1$
    when generating traces and $K{=}1000$ for evaluation.
    \item \emph{Learned $\widehat V$.} We train a small classifier that reads a prefix $s$ and
    predicts $\widetilde V(s)$. To get a training label for $s$, we let the \modelc{} complete $s$
    once and record whether the result is valid. The classifier is trained on these labels with
    binary cross-entropy, as in Eq.~(4) of \citet{rohatgi2025vgb}. At inference time, we call the
    classifier once instead of sampling $K$ completions.
\end{itemize}

\noindent\textbf{Trace format.}
A trace starts with the board, then lists the steps of the search, and ends with \texttt{<EOS>}.
Each step takes two tokens. The first is the verifier score at that step, rounded to one of $10$
levels and written as a token \texttt{<V$_q$>}. The second is the harness's decision: the action it
commits, or \BT{}. Because the score comes right before each decision, the student can see what
made the harness backtrack. In total the vocabulary has $27$ tokens:
\[
    \textstyle
    \underbrace{\texttt{<board>}}_{\text{prompt}}\;\Big|\;
    \underbrace{\texttt{<V$_q$>}\;\texttt{<action\_or\_BT>}\;\cdots\;\texttt{<V$_q$>}\;\texttt{<action\_or\_BT>}}_{\text{search log (2 tokens / step)}}
    \;\Big|\;\underbrace{\texttt{<EOS>}}_{\text{end}}.
\]
In the ablations we compare four versions of these traces. They differ in two ways: the harness
scores prefixes with either the learned $\widehat V$ or the MC oracle, and the score tokens are
either kept in the trace or removed. Without score tokens, the student sees only the actions and the
backtracks, and has to work out by itself when to backtrack.

\subsubsection{A* search (LM-pruned)}
\label{app:astar}

The A* harness is ordinary A* search, with the \modelc{} used to prune it. Like any A* search, it
keeps a queue of board states ordered by $f=g+h$, where $g$ counts the moves made so far and $h$ is
the Manhattan-distance heuristic. At each step it takes the state with the smallest $f$ out of the
queue. Plain A* search would then put every state that is one legal move away into the queue. Our
harness puts in only two of them: the states reached by the two moves that the \modelc{} is most
likely to make there. With this pruning the harness still finds optimal plans, and its traces are
about $36\%$ shorter than those of plain A* search.

\noindent\textbf{Trace format.}
The trace lists every state that the search creates (\texttt{<CREATE>}, followed by the state's
coordinates and its values of $g$, $h$ and $f$) or closes (\texttt{<CLOSE>}, followed by its
coordinates), and it ends with the plan after \texttt{<PLAN>}. With these markers and the number
tokens, the vocabulary has $45$ tokens:
\[
    \textstyle
    \small
    \underbrace{\texttt{<board>}}_{\text{prompt}}\;\Big|\;
    \underbrace{\big(\texttt{<CREATE>}\,pr\,pc\,br\,bc\,g\,h\,f \;\mid\; \texttt{<CLOSE>}\,pr\,pc\,br\,bc\big)^{*}}_{\text{search log}}
    \;\Big|\;\underbrace{\texttt{<PLAN>}\,a_1 a_2\cdots\texttt{<EOS>}}_{\text{final plan}}.
\]
At test time the student writes the whole trace by itself with greedy decoding, without calling
the \modelc{}, and we check whether the plan after \texttt{<PLAN>} solves the puzzle.

\subsection{Experimental Details}
\label{app:exp-details}

\noindent\textbf{Architecture.}
Unless we say otherwise, the base model, the model inside the teacher and the same-size student are
the same small Transformer, with $2$ layers, width $128$ and $4$ attention heads (about $0.68$M
parameters). Table~\ref{tab:scales} lists the other sizes we use when we vary the size of the
student.

\noindent\textbf{Dyck.}
Each sequence is a string of $64$ brackets of four kinds, and it is valid if every opening bracket
is closed by the matching bracket in the right order. The ID and OOD sequences follow the same
rules but use the four kinds of bracket with different frequencies: $0.1$, $0.2$, $0.3$ and $0.4$
in distribution, and the reverse out of distribution. We train the base model on $300$K valid ID
sequences, using AdamW with learning rate $3\times10^{-4}$ for $3000$ iterations. When the harness
backtracks on Dyck, it deletes the last $4$ tokens and lets the base model write them again, and it
backtracks at most $50$ times. Two verifiers can tell it when to backtrack. The learned one is a
linear classifier that reads the hidden state of the base model and predicts whether the last token
is a mistake; it is correct on $87.3\%$ of tokens. The oracle checks the prefix directly, so it is
always correct.

\noindent\textbf{Sokoban.}
Each puzzle is a $5\times5$ grid with one player, one box, one target and at most three walls, and
the goal is to push the box onto the target. The player moves up, down, left or right, and \BT{}
undoes the last move. The model reads a puzzle as its $25$ cells, row by row, and then writes the
moves. Solutions are between $3$ and $15$ moves long, and the two splits (ID and OOD) differ in how long they
tend to be: short solutions are the most common in distribution, and long ones are the most common
out of distribution. The base model is trained in the same way as on Dyck, and we test on $300$
held-out puzzles from each split. The oracle verifier runs a breadth-first search to check whether
the puzzle can still be solved from the current state.

\noindent\textbf{Students.}
A student learns from the teacher's traces in one of two ways. Either we train it from scratch on
the traces alone, or we start from the base model and fine-tune it on a mix of traces and clean
data. Unless we say otherwise, the mix is $70\%$ clean data and fine-tuning runs for $5$K
iterations. At test time the student is on its own, with no harness and no verifier.

\subsection{Full Ablation Grids}
\label{app:cannot-recover-ablation}

This section collects the full sweeps behind the ablations in
Section~\ref{sec:cannot-recover-ablation}. We ran every sweep on both tasks and with all four
versions of the VGB traces, and the figures show a few representative cases. In every panel, the
dashed lines mark the teacher.

\subsubsection{Fine-tuning the Base Model}
\label{app:ft-grids}

Here the student starts from the \modelc{} and is fine-tuned on a mix of traces and clean data. A
default run uses all of the traces, adds clean data until it makes up $70\%$ of the mix, and trains
for $5$K iterations. In each sweep we change just one of these three choices and leave the other two
as they are. We try clean-data shares from $0\%$ to $100\%$, training runs of $3$K, $5$K or $8$K
iterations, and $1$K, $5$K or $20$K traces. We also change which puzzles the traces come from: instead of ID puzzles only, we
draw them from the mix $P_\alpha=(1-\alpha)P_{\text{ID}}+\alpha P_{\text{OOD}}$ for
$\alpha\in\{0.25,0.5,0.75\}$, and repeat the clean-data sweep for each $\alpha$.
Figure~\ref{fig:vgb-ft-ablation} shows all four sweeps for VGB, and Figure~\ref{fig:app-astar-ft}
shows the first three for A* search.

The first question is whether the student simply needs the right amount of clean data. Clean data
does help, but only up to a point. On VGB, accuracy climbs quickly as the first $30\%$ of clean data
is added, and after that it hardly moves, staying at about the level of the \modelc{}
(Figure~\ref{fig:vgb-ft-ablation}a). In other words, clean data repairs the damage that training on
traces alone does to the student, but it does not lift the student above the \modelc{}. The A*
traces are harder to learn, and without any clean data the A* student fails almost completely
(Figure~\ref{fig:app-astar-ft}a).

Training longer or on more traces does not help much either. Going from $3$K to $8$K iterations adds
only a few points, and twenty times more traces barely changes the student
(Figure~\ref{fig:vgb-ft-ablation}b,c), so the gap is unlikely to come from too little training or
too little data.

Finally, the gap could come from which puzzles the traces are drawn from. Mixing OOD puzzles into
training raises OOD accuracy a little, as one would expect, but the student ends up about as far
below the teacher as before (Figure~\ref{fig:vgb-ft-ablation}d--f).

\subsubsection{Training from Scratch}
\label{app:fs-grids}

We run a similar set of checks for students trained from scratch at the size of the \modelc{}. We
train with four different random seeds, to see how much the results vary. We train on traces from
OOD puzzles, which tells us whether the student learns the search itself or only the statistics of
the ID puzzles. And we draw the training puzzles from the same mix $P_\alpha$ as above.
The results are shown in Figure~\ref{fig:app-fs}.

\section{Supplementary Details for Section~\ref{sec:closing-gap}}
\label{app:validation}

This appendix gives more details on the three experiments of Section~\ref{sec:closing-gap}.

\subsection{Model-Size Sweep (Section~\ref{sec:gap-size})}
\label{app:validation-size}

\noindent\textbf{Setup.}
In this sweep we change the size of the student and keep everything else fixed. Every point is a
new student, trained from scratch on the same traces as in Section~\ref{sec:cannot-recover} and
tested on $200$ ID and $200$ OOD puzzles. Table~\ref{tab:scales} lists the $12$ sizes, which range
from $0.10$M to $5.65$M parameters.

\noindent\textbf{Results.}
The main text shows VGB on Dyck (Figure~\ref{fig:scale-bt}b), and
Figure~\ref{fig:app-modelsize-best} shows the two Sokoban settings. On Sokoban, too, larger students
do better, but how far they get depends on the harness. With A* search, the gap to the teacher
shrinks steadily as the student grows, from $0.65$ to $0.13$, yet no student reaches the teacher.
With VGB, the largest students go past the teacher, both ID and OOD.

\noindent\textbf{Each setting on its own.}
The curves above take the best run at each size over all clean-data ratios and checkpoints.
Figure~\ref{fig:app-modelsize-ratio} shows
each clean-data ratio and each checkpoint separately. Accuracy still goes up with size in every
panel, both ID and OOD.

\subsection{Simple-BT Distillation (Section~\ref{sec:gap-easy})}
\label{app:validation-easy}

\noindent\textbf{Setup.}
Here the teacher uses a much simpler harness than VGB or A* search. It runs the \modelc{} on Dyck,
and whenever a verifier flags the last token as a mistake, it deletes the last $4$ tokens and lets
the \modelc{} write them again. The verifier is either the learned linear classifier or the oracle
from Appendix~\ref{app:exp-details}. The student starts from the \modelc{} and is fine-tuned on a
mix of clean data and teacher traces, where we keep only the traces with exactly one backtrack. We
also vary where the traces come from, by drawing the sequences from a mix $P_\alpha$ of the ID and
OOD distributions.

\noindent\textbf{Results.}
Table~\ref{tab:app-bt} lists the accuracy of every student in two modes. In the first (one pass),
the student writes the whole sequence from left to right with top-$p$ sampling ($p{=}0.9$), without
any backtracking. In the second, it samples in the same way, but whenever it writes a \BT{} token,
the last $4$ tokens are deleted and it continues from there, up to $50$ times. For comparison, the table also includes students
fine-tuned on synthetic backtracking data, in which we insert errors taken from OOD data, chosen at
random, or taken from ID data. In distribution there is little to gain, since the distilled students
already match the \modelc{} in one pass. The interesting case is OOD. There, letting the student use
its own \BT{} tokens and drawing the traces from a slightly OOD mix ($\alpha{=}0.25$) together lift
it a little above the \modelc{}. It still stays well below the teacher that uses the oracle
verifier.

\subsection{Long-CoT on Modular Arithmetic (Section~\ref{sec:gap-cot})}
\label{app:validation-cot}

\noindent\textbf{Task.}
We follow the setup of \citet[Alg.~1]{feng2023cot}. Each problem is an arithmetic expression modulo
$11$ with $n_{\text{ops}}$ operators, and it is solved one step at a time. Each step simplifies the
leftmost part of the expression that can be simplified, so every problem has exactly one chain of
steps. To control how much chain-of-thought the model sees, we train a $3$-layer Transformer (about
$2.5$M parameters) on the problem, then only the first $L$ steps of its chain, then the answer, and
we vary $L$ from $0$ to $n_{\text{ops}}$. We test on $1000$ held-out problems.

\noindent\textbf{Harness.}
The harness writes the solution one step at a time, where a step is one simplification of the
expression. At each step, the model proposes $16$ candidates for the next step. An oracle verifier
gives each candidate a score of $1$ if the chain so far, with this candidate added, agrees with the
true chain, and $0$ otherwise. If some candidate passes, the harness adds it to the chain and moves
on. If none passes, the harness backtracks: it removes the last step it added and tries that
position again, this time without allowing the removed step, so that it does not make the same
mistake twice.
Figure~\ref{fig:app-cot}a illustrates this harness. Figure~\ref{fig:app-cot}b shows the results on
an easier version of the task, with $n_{\text{ops}}{=}10$; the main text (Figure~\ref{fig:cotlen})
shows $n_{\text{ops}}{=}20$.

\section{Supplementary Details for Section~\ref{sec:theory-gap}}
\label{app:proofs}

This appendix proves Lemma~\ref{lem:logcot} and Theorem~\ref{thm:main}.
Throughout, circuit uniformity means logspace uniformity. The input length is
$n=\Theta(N^2)$, $\Sigma$ is fixed and finite, and
$T_c(n)=\lceil c\log_2(n+1)\rceil$ for a fixed integer $c\ge1$.

\subsection{The local query and the hardness of reachability}
\label{app:reachability}

For $A\in\{0,1\}^{N\times N}$ and $u,v\in[N]$, the local predicate is
\[
  \textstyle
  E(A,u,v)=\bigvee_{i,j\in[N]}\bigl(A_{ij}\land[u=i]\land[v=j]\bigr).
\]
Writing each equality as a conjunction of binary literals gives a constant-depth
Boolean circuit with polynomially many gates and wires. A logarithmic-space procedure
can generate the gates and their wiring by enumerating $(i,j)$, so $E\in\AC^0$.
The notation $E_A(u,v)$ fixes the graph and makes the two queried vertices explicit.

Directed $\STCON$ is $\NL$-complete under uniform projection translations, and
hence under uniform $\AC^0$ many-one reductions
\citep[Corollary~4.7]{scaffold_immerman1987}. The standard configuration-graph
construction explains the strength of the reduction. For a fixed nondeterministic
logspace machine and input $x$, its polynomially many configurations become vertices;
an edge records one legal machine transition. After fixing the two configurations,
the legality of that transition depends on at most one input bit. In a standard
normal form, each output adjacency bit is therefore a constant, an input bit, or its
negation, and its wiring is computable in logspace. The source and target encode the
initial configuration and a designated accepting configuration.

This reduction strength is what Theorem~\ref{thm:main} uses. An arbitrary logspace
reduction could itself perform work outside $\TC^0$, whereas a uniform $\AC^0$
reduction can be composed with a $\TC^0$ decision circuit while preserving constant
depth, polynomial size, and logspace uniformity.

\subsection{Proof of Lemma~\ref{lem:logcot}}
\label{app:logcot}

\begin{proof}
Let $L\in\CoT_{T_c(n)}(\Sigma,\TC^0)$, with predicates $V_{n,i}$ and
$B_n$ as in Definition~\ref{def:cot}. Let $d$ and $p(n)$ be their common depth
and polynomial size bounds. Since the alphabet and $c$ are fixed,
\[
  \textstyle
  M(n):=|\Sigma|^{T_c(n)}
  \le |\Sigma|(n+1)^{c\log_2|\Sigma|}=\poly(n).
\]
Enumerate all transcripts as $z^{(1)},\ldots,z^{(M(n))}$. For each transcript,
form the circuit
\[
  \textstyle
  \mathrm{Acc}_j(x)=B_n\bigl(x,z^{(j)}\bigr)
  \land\bigwedge_{i=1}^{T_c(n)}
  V_{n,i}\bigl(x,z^{(j)}_{<i},z^{(j)}_i\bigr).
\]
Each predicate is evaluated after hardwiring its transcript inputs. This does not
increase its depth or size. Add one AND gate for each transcript and one final OR:
\[
  \textstyle
  C_n(x)=\bigvee_{j=1}^{M(n)}\mathrm{Acc}_j(x).
\]
The resulting circuit has depth at most $d+2$ and size at most
$M(n)(T_c(n)+1)p(n)+O(M(n))$, which is polynomial in $n$.

For uniformity, the integer $c$ is fixed and $T_c(n)$ is logspace-computable.
Both a transcript and its index use $O(\log n)$ bits, so a logarithmic-space
generator can enumerate the transcripts, produce their hardwired bits, and invoke
the jointly uniform generators for $V_{n,i}$ and $B_n$. It can then add the AND
and OR wiring using counters of $O(\log n)$ bits. Thus $(C_n)$ is logspace-uniform.
By Definition~\ref{def:cot}, $C_n(x)=1$ exactly when $x\in L$, establishing
$L\in\TC^0$.
\end{proof}

This is the transcript-enumeration argument used for logarithmic-length transformer
CoT by \citet[Theorem~4 and Appendix~E]{scaffold_merrill2025depth}. Joint uniformity
and the explicit computable ceiling ensure that the construction also satisfies
the uniformity convention of this paper. The argument uses the polynomial number
of possible transcripts; it does not establish a sharp threshold for the power of
longer chains of thought.

\subsection{Proof of Theorem~\ref{thm:main}}
\label{app:separation}

\begin{proof}
\emph{Part (i).} The preceding construction shows $E\in\AC^0$. Run breadth-first
search using $E_A$ to test existence of edges. Store a visited set $R$, a current frontier $F$,
and a next frontier $F'$ as $N$-bit characteristic vectors. Initially,
$R=F=\{s\}$. At each iteration, set
\[
  \textstyle
  F'=\{v\in[N]\setminus R:\ E_A(u,v)=1\text{ for some }u\in F\},
\]
by querying the pairs in $F\times([N]\setminus R)$, and update
$R\leftarrow R\cup F'$ and $F\leftarrow F'$. Stop when $F$ is empty, and
accept exactly when $t\in R$.

After $k$ iterations, $R$ contains exactly the vertices at directed distance at
most $k$ from $s$, and $F$ contains those first reached at that distance. This
invariant proves correctness, including the case $s=t$. A vertex enters the
frontier at most once, so each ordered pair is queried at most once, giving at
most $N^2$ oracle calls. The three characteristic vectors and the loop counters
use $3N+O(\log N)=O(N)$ bits. The loop has at most $N$ iterations and all its
bookkeeping takes polynomial time.

\emph{Part (ii).} Suppose, for some fixed $\Sigma$ and integer $c\ge1$, that
$\STCON\in\CoT_{T_c(n)}(\Sigma,\TC^0)$. Lemma~\ref{lem:logcot} then
gives $\STCON\in\TC^0$. For any $L\in\NL$, let $f_L$ be the uniform $\AC^0$
many-one reduction from $L$ to $\STCON$ supplied by projection completeness
\citep{scaffold_immerman1987}. Composing a $\TC^0$ circuit for $\STCON$ with the
circuits for the output bits of $f_L$ gives a circuit for $L$. The composition
preserves constant depth, polynomial size, and logspace uniformity. Hence
$\NL\subseteq\TC^0$, contradicting the assumption.
\end{proof}

\section{Supplementary Details for Section~\ref{sec:agentic}}
\label{app:agentic}

This appendix gives more details on the experiments in Section~\ref{sec:agentic}. We first describe more distillation experiments on Qwen3-1.7B (Appendix~\ref{app:agentic-branching}), and then the experiment on the QED-Nano RL checkpoint.

\subsection{Branching distillation on Qwen3-1.7B}
\label{app:agentic-branching}

\noindent\textbf{Distillation settings.}
The top row of Figure~\ref{fig:app-branching-ablation} checks whether a better distillation recipe
would close the gap, using the $\mathbf{ks}{=}[3]$ tree. We vary three things. The first is how much
clean data we mix into the Branching traces, where clean data means in-distribution solutions that
the \modelc{} writes itself. The second is the number of traces. The third is the length of training, and the student stops
improving after two epochs. None of these settings brings the student up to the teacher.

\noindent\textbf{All tree shapes.}
A tree with $d$ children per node and depth $L$ calls the model $\sum_{i=0}^{L}d^{i}$ times, so the
shapes we try range from $3$ calls for $(2,1)$ to $85$ for $(4,3)$. The bottom row of
Figure~\ref{fig:app-branching-scaling} shows the number of calls for each shape (d) and the accuracy
of the teacher and the student (e; the main text shows the same panel). As the teacher makes more
calls it gets better almost every time, while the student shows no clear trend.

\noindent\textbf{More capacity.}
If the student falls short because it is too small, giving it more capacity should help. We try two
ways of doing this (Figure~\ref{fig:app-branching-capacity}). Training more parameters of the same
model helps only a little: with a higher LoRA rank, or even with full fine-tuning, the student stays
well below the teacher (panel b). Switching from the $1.7$B to the $4$B model, on the other hand,
takes the student past the teacher (panel a). This suggests that the student needs a larger model,
and that training more parameters of the same model is not enough.

\subsection{The QED-Nano checkpoint and its harness}
\label{app:qed}

QED-Nano \citep{qednano2026} is a $4$B model for writing Olympiad-level proofs. It starts from
Qwen3-4B-Thinking and is trained in two stages: first supervised fine-tuning on proofs
written by the much larger DeepSeek-Math-V2, then reinforcement learning on the
\texttt{FineProofs-RL} problems, with rewards computed from grading rubrics. The harness is only
added at test time and is not trained. We use the released checkpoint, which is the model after both
stages.

The QED-Nano paper reports that its harness, Recursive Self-Aggregation (RSA), helps a lot: on
IMO-ProofBench it raises accuracy from $40.0\%$ to $56.9\%$, at the cost of about $22$ times more
tokens. RSA works with a pool of $n$ proofs. The model first writes $n$ proofs independently. Then,
for $T$ rounds, it replaces the pool with new proofs, each written by merging $k$ proofs drawn at
random from the current pool. Finally, the proofs are paired up, the
model reads each pair and keeps the better proof, and this repeats until only one proof is left.
That proof is the final answer. The same model plays every role, with a
different prompt for each.

This harness is far too large to distill. With $39$ calls, one RSA trace is about $1.5$M tokens
long, several times longer than any context we could train on. We therefore test the simpler
Branching harness, which writes a few independent proofs and merges them in a few steps. Its traces
are short enough that distilling them is at least possible, and we ask whether the harness still
helps at these smaller sizes.

\subsection{Corpus}

We test on $300$ problems from \texttt{FineProofs-RL}, the $5{,}227$ problems used in the model's
reinforcement-learning stage. Using this set has two advantages. The QED-Nano authors already
removed the problems that overlap with IMO-ProofBench, the benchmark the harness was tuned on, and
every problem comes with a grading rubric, which we need to score a free-form proof.

We only keep problems of medium difficulty, namely those on which the \modelc{}'s average reward
over $128$ attempts lies between $0.10$ and $0.60$. On harder problems the model
almost never writes a correct proof, so the harness has mostly wrong proofs to merge and little
chance to improve on them. On easier problems a single call is usually already correct, so there
is little left for the harness to add. From this range we take $300$
problems evenly spaced by difficulty, so that every run uses the same problems.

\subsection{Shapes}

Table~\ref{tab:qed-shapes} lists the five tree shapes we run. We describe a shape by a pair $(r,b)$.
The model first writes $b^r$ proofs; then, in each of $r$ rounds, it merges every group of $b$
proofs into one, until a single proof is left. We turn off the final pairwise comparison, so a trace
contains only the tree. The shapes $(1,6)$ and
$(2,2)$ cost the same, $7$ calls each, but merge the proofs in different ways. With $(1,6)$, the
model writes $6$ proofs and merges all of them at once. With $(2,2)$, it writes $4$ proofs, merges
them in pairs into $2$ proofs, and then merges those $2$ into one. Comparing the two therefore shows
whether the shape of the tree matters when the cost is held fixed.

\subsection{Generation and grading}

\noindent\textbf{Generation.}
We serve the model with vLLM, allow up to $65{,}536$ tokens per call, and sample once per problem at
temperature $1.0$. We use the prompts from the QED-Nano repository without changes.

\smallskip
\noindent\textbf{Base and harness answers.}
Both answers come from the same run. The base answer is the output of the first call, which is just
one ordinary attempt at the problem, so the baseline costs nothing extra. The harness answer is the
output at the root of the tree. Because the two answers share the problem, the run and the sampling
settings, we can compare them problem by problem. In both cases we grade only the final answer and
not the model's thinking.

\smallskip
\noindent\textbf{Grading.}
A local \texttt{gpt-oss-20b} judge scores each answer out of $7$ points following the problem's
rubric, and we count an answer as correct if it scores at least $6$. All the percentages we report
are these pass rates. Before grading we cut answers off at $100{,}000$ characters, because a few
broken answers are extremely long and would otherwise make the grading call fail. For the gain
of the harness over the base answer, we give a $95\%$ confidence interval from a paired bootstrap.

\subsection{Results and caveats}

The results are in Figure~\ref{fig:agentic-scaling}b and Figure~\ref{fig:app-branching-capacity}c.
For the four small trees, which make $5$ to $7$ calls, the harness is not clearly better than a
single call: the confidence interval of the gain contains zero. Only the largest tree, $(3,3)$ with
$40$ calls, does clearly better than a single call. But its trace is about $1.4$M tokens long, far
more than the $262$K-token context window, so a student could not be trained on it. Two caveats apply.

\smallskip
\noindent\emph{(i) We do not filter by trace length.} How long a trace gets depends on the harness
itself: when a merge runs into the token limit, the trace becomes too long and the answer also
scores zero. If we dropped the problems with long traces, we would mostly be removing
problems on which the harness failed, and the harness would look better than it really is. We
therefore keep every problem.

\smallskip
\noindent\emph{(ii) We use a single run.} Each problem is run once at temperature $1.0$ and graded
once.

\makeatletter\loop\unless\ifx\@deferlist\@empty\null\newpage\repeat\makeatother
\end{document}

%% file: math_commands.tex
\usepackage{amsmath,amsfonts,bm}

\def\eqref#1{equation~\ref{#1}}

\def\1{\bm{1}}

\DeclareMathAlphabet{\mathsfit}{\encodingdefault}{\sfdefault}{m}{sl}
\SetMathAlphabet{\mathsfit}{bold}{\encodingdefault}{\sfdefault}{bx}{n}